\documentclass[11pt]{article}
\usepackage{bbm}
\usepackage{eqnarray,amsmath,amsfonts,amsthm,mathrsfs}
\usepackage{color}
\usepackage{bm}
\usepackage{amssymb}
\usepackage[dvips]{graphicx}
\usepackage{epsfig}
\usepackage{epsf}
\usepackage{float}
\usepackage{subfigure}
\usepackage{amsfonts,amsmath,amsthm,amssymb,graphicx,float,fancyhdr,multirow,hyperref}
\usepackage{booktabs,longtable,authblk}
\usepackage{mathrsfs,hhline}

\usepackage{fancyhdr}
\usepackage[top=2.5cm, bottom=2.5cm, left=3cm, right=3cm]{geometry}
\usepackage{graphicx}
\newtheorem{theorem}{Theorem}
\newtheorem{assumption}{Assumption}
\newtheorem{corollary}{Corollary}
\newtheorem{definition}{Definition}

\newtheorem{lemma}{Lemma}[section]
\newtheorem{proposition}{Proposition}[section]
\newtheorem{remark}{Remark}[section]
\newtheorem{example}{Example}[section]

\allowdisplaybreaks[4]

\def\begeqn{\begin{equation}}
\def\endeqn{\end{equation}}
\def\begth{\begin{theorem}}
\def\endth{\end{theorem}}
\def\begprop{\begin{proposition}}
\def\endprop{\end{proposition}}
\def\begcor{\begin{corollary}}
\def\endcor{\end{corollary}}
\def\begdef{\begin{definition}}
\def\enddef{\end{definition}}
\def\beglemm{\begin{lemma}}
\def\endlemm{\end{lemma}}
\def\begexm{\begin{example}}
\def\endexm{\end{example}}
\def\begrem{\begin{remark}}
\def\endrem{\end{remark}}
\def\begassum{\begin{assumption}}
\def\endassum{\end{assumption}}

\def\N{\mathbb{N}}
\def\R{\mathbb{R}}

\def\E{\mathcal{E}}

\def\H{\mathcal{H}}

\def\R{\mathbb{R}}

\def\bE{{\mathbb{E}}}

\def\RR{{\mathbb R}}

\title{Fast and Strong Convergence of Online Learning Algorithms$^\dag$\footnotetext{\dag~The work described in this paper is supported partially by the National Natural Science Foundation of China (Grants No.11401524, 11531013, 11571078, 11631015). Lei Shi is also supported by the Joint Research Fund by National Natural Science Foundation of China and Research Grants Council of Hong Kong (Project No. 11461161006 and Project No. CityU 104012) and Zhuo Xue program of Fudan University. Part of the work was carried out while Zheng-Chu Guo was visiting Shanghai Key Laboratory for Contemporary Applied Mathematics. The corresponding author is Lei Shi. Email addresses: guozhengchu@zju.edu.cn (Z. C. Guo), leishi@fudan.edu.cn (L. Shi).}}

\author{Zheng-Chu Guo$^1$ and Lei Shi$^2$\\
\small $^1$ School of Mathematical Sciences, Zhejiang University, Hangzhou 310027, P. R. China \\
\small $^2$ Shanghai Key Laboratory for Contemporary Applied Mathematics,\\
\small School of Mathematical Sciences, Fudan University, Shanghai 200433, P. R. China\\}

\date{}
\begin{document}

\maketitle

\begin{abstract}
In this paper, we study the online learning algorithm without explicit regularization terms. This algorithm is essentially a stochastic gradient descent scheme in a reproducing kernel Hilbert space (RKHS). The polynomially decaying step size in each iteration can play a role of regularization to ensure the generalization ability of online learning algorithm. We develop a novel capacity dependent analysis on the performance of the last iterate of online learning algorithm. The contribution of this paper is two-fold. First, our nice analysis can lead to the convergence rate in the standard mean square distance which is the best so far. Second, we establish, for the first time, the strong convergence of the last iterate with polynomially decaying step sizes in the RKHS norm. We demonstrate that the theoretical analysis established in this paper fully exploits the fine structure of the underlying RKHS, and thus can lead to sharp error estimates of online learning algorithm.
\end{abstract}


\section{Introduction}\label{section: introduction}
Analyzing and processing large-scale data sets is becoming ubiquitous in the era of big data. How to reduce computational complexity and memory requirement is the pivotal consideration for designing learning algorithms for big data. Although the batch learning algorithms are well understood in theory and widely used in various applications, as the sample size $T$ getting larger, it is still challenging to find effective solutions in practice. Generally, if the optimization process in batch learning involves matrix inversion or decomposition, it usually requires $\mathcal{O}(T^2)$ memory and $\mathcal{O}(T^3)$ time. Such scalings of algorithmic complexity seriously limit the performance of batch learning for big data. Contrast to the batch learning which is required to tackle the whole data set in a batch, online learning processes the data one by one and updates the output in time. Inspired by the gradient descent method, online leaning iteratively builds an unbiased estimate of the true gradient upon the arrival of a new data and uses this information to guide the learning process. One distinguished feature of online learning algorithm as compared to its batch counterpart is the prominent computational speed-up. In a standard implementation, online learning algorithm merely requires linear complexity $\mathcal{O}(T)$ since its gradient calculation only involves one single sample. Recently, online learning has received more attention due to its applications in training deep neural networks \cite{nemirovski2009robust,sutskever2013importance}. Despite its widespread applications, the theoretical understanding of the online learning algorithms is still not satisfactory. How to improve convergence analysis of online learning with some additional capacity information is stated as an open problem in \cite{Ying2008,Rosasco2014}. In this paper, we make a further step to solve this problem by establishing a novel capacity dependent error analysis, which leads to fast and strong convergence of online learning algorithm.

We consider online learning in the setting of regression. The goal of regression is to infer a functional relation from random samples that can make predictions about future
observations. More specifically, let $X$ be a compact metric space and $Y \subseteq \mathbb{R}$, $\rho$ be a Borel probability
distribution on $Z=X \times Y$. For a function $f:X\to Y$ and $(x,y)\in Z$, where the value $f(x)$ represents the prediction of $y$
based on $x$, the prediction error incurred is characterized by the least-squares loss $(f(x)-y)^2$. The regression problem aims
at estimating an ideal prediction model which minimizes the generalization error
$$ \E(f)=\int_{Z} (f(x)-y)^2 d \rho$$
over all measurable functions. This estimation target is referred to as the regression function and is given by
\begin{equation*}\label{frho}
f_{\rho}(x)=\int_{Y}yd\rho(y|x), \quad x\in X,
\end{equation*}
where $\rho(\cdot|x)$ is the conditional distribution at $x$ induced by $\rho$. In the framework of learning theory, $\rho$ is
unknown and one estimates $f_{\rho}$ based on samples drawn from $\rho$ independently. Regression has been extensively studied in machine learning and statistical inference. There is a large literature on regression in learning theory, e.g., see \cite{cucker2007learning,Steinwart2008} and references therein.

We focus on the kernel-based learning scenario in this work, which was first developed within the statistics community for nonparametric regression using splines and Sobolev spaces \cite{wahba1990,Gu2002}. More generally, the learning process is implemented in a reproducing kernel Hilbert space (RKHS) $\mathcal{H}_{K}$ associated with a Mercer kernel $K: X \times X \to \R$. Such a kernel is a continuous, symmetric, and positive semi-definite function over $X \times X$. The space $\mathcal{H}_{K}$ is the completion of the span $\{K_{x}=K(\cdot, x): x\in X\}$ with respect to the inner product $\langle \cdot, \cdot\rangle_K$ given for fundamental functions by $\langle K_{x}, K_{x'}\rangle_K =K(x,x')$ \cite{Aron}. It is well-known that a Mercer kernel $K$ uniquely induces an RKHS ${\cal H}_K$, and the reproducing property
\begin{equation}\label{reproducingproperty}
f(x)= \langle f, K(\cdot,x) \rangle_K
\end{equation} holds for all $x\in X$ and $f\in {\cal H}_K$.

Based on the sequential sample $\{z_t=(x_t, y_t)\}_{t\in\N}$ drawn independently from $\rho$, we approximate the regression function by an online learning algorithm without explicit regularization terms. This algorithm is essentially a stochastic gradient descent scheme in an RKHS $\mathcal{H}_{K}$. Using the reproducing property (\ref{reproducingproperty}), the gradient of generalization error  $\mathcal{E}(\cdot)$ at $f\in {\cal H}_{K}$ is given by $2L_K(f-f_{\rho})$, where $L_K: {L}^2_{\rho_X} \to {L}^2_{\rho_X}$ is an integral operator defined with the Mercer kernel $K$ and the marginal distribution $\rho_X$ of $\rho$ by
\begin{equation}\label{operatorLK}
L_K f (\cdot) = \int_X K(\cdot,x) f(x) d\rho_X (x), \quad \forall f\in {L}^2_{\rho_X}.
\end{equation} Here ${L}^2_{\rho_X}$ denotes the Hilbert space of functions from $X$ to $Y$ square-integrable with respect to $\rho_X.$ Given a sample $z_t=(x_t,y_t)$, an unbiased estimator of the true gradient is $2(f(x_t)-y_t)K_{x_t}$. Then the online learning algorithm is expressed as, $f_1=0,$ and
\begin{equation}\label{algorithm}
f_{t+1}=f_t-\eta_t(f_t(x_t)-y_t)K_{x_t},\qquad t\ge 1,
\end{equation}
where $\eta_t>0$ is the step size parameter. Learning in a probably infinite dimensional RKHS will inevitably lead to the phenomenon of over-fitting, which should be prevented by some regularization techniques. The choice of the step size sequence $\{\eta_t\}_{t\in \mathbb{N}}$ can play a role of regularization to ensure the generalization ability of algorithm (\ref{algorithm}).

One usually measures the performance of the iterate $f_t$ by the excess generalization error defined as $\E(f_t)-\E(f_\rho)$. Simple calculations show that $\E(f_t)-\E(f_\rho)=\|f_t-f_\rho\|_\rho^2,$ where $\|\cdot\|_{\rho}$ denotes the norm in the space ${L}^2_{\rho_X}$ induced by the inner product $\langle f, g\rangle_{\rho}=\int_X f(x) g(x) d\rho_X(x)$. One may also consider $\|f_t-f_\rho\|_K$ to measure the performance of $f_t,$ i.e., the strong convergence of $f_t$ to $f_\rho$ in $\H_K$. As pointed out in \cite{Smale2007}, the convergence in $\H_K$ implies the convergence in $C^s(X)$ if $K\in C^{2s}(X \times X)$ for some integer $s\in \mathbb{N}$, where $C^s(X)$ is the space of all functions on $X \subset \RR^p$ whose partial derivatives up to order $s$ are continuous with $\|f\|_{C^s(X)}=\sum_{|\alpha|\le s}\sup_{x\in X}|D^\alpha f (x)|.$ Hence the convergence in ${\cal H}_K$ is much stronger, which ensures that $f_t$ can not only approximate the regression function itself, but also approximate its derivatives.

In this paper, we focus on the performance of the last iterate of algorithm (\ref{algorithm}) with polynomially decaying step sizes both in $L_{\rho_X}^2$ and $\H_K$. In general, there are two types of step size: the first one is the decreasing step size, e.g., polynomially decaying step size of the form $\eta_t=\eta_1 t^{-\theta}$ with $\theta\in(0,1)$ and some $\eta_1>0$; the other one is constant step size of the form $\eta_t=\eta(T)$ depending on the total number of iterates $T$, i.e., the sample size processed by online learning algorithm after $T$ iterations. Both types of step sizes can serve as an implicit regularization. In some real-time learning system, e.g., the real-time traffic control system, the data arrive in a sequential manner, determining the sample size is unpractical. Therefore, we are more concerned with the performance of algorithm (\ref{algorithm}) equipped with decreasing step size, since the choice of constant step size requires knowing the sample size $T$ in advance. Moreover, satisfactory convergence rates are derived both in $L^2_{\rho_X}$ norm and ${\cal H}_K$ norm for algorithm (\ref{algorithm}) with constant step sizes in \cite{Ying2008}. However, for the case of polynomially decaying step sizes, there still lacks a convergence analysis in $\H_K$ yet, and faster convergence rate in $L^2_{\rho_X}$ is also in demand. In this work, we aim to fill this gap by developing an elegant capacity dependent analysis on online learning algorithm. We would like to stress that our analysis established here is specific to the last iterate of algorithm (\ref{algorithm}). There has already been several literature on how to improve the convergence results by averaging schemes in online learning, e.g., see \cite{Dieuleveut&Bach2016}. Taking average of the outputs in each iterate may generate more robust solutions \cite{nemirovski2009robust}, but it also slows down the training speed in the practical implementations \cite{rakhlin2012making}.

The remainder of this paper is organized as follows. We present main results in Section \ref{section: main results}. Discussions and comparisons with related work are given in Section \ref{section: related work}. The proofs of main results are given in Section \ref{section: convergenceinL2} and Section
\ref{section: covergenceinHK}.

\section{Main Results}\label{section: main results}
We begin with some basic notations and assumptions. In the sequel, we assume that the marginal distribution $\rho_X$ is non-degenerate, which ensures $\rho_X(X')>0$ for each non-empty open subset $X' \subseteq X$. Recall that the integral operator $L_K$ is defined by (\ref{operatorLK}). Since $K$ is a Mercer kernel and $X$ is compact, the operator $L_K$ is self-adjoint, positive and trace class on $L^2_{\rho_{\cal X}}$, as well as restricting on ${\cal H}_K$. The compactness of $L_K$ implies the existence of an orthonormal eigensystem $\{\sigma_k, \phi_k\}_{k\in \mathbb{N}}$ in $L^2_{\rho_X}$, where the eigenvalues $\{\sigma_k\}_{k\in \mathbb{N}}$ (with geometric multiplicities) are non-negative and arranged in decreasing order. Then for any $r>0$, the $r-$th power of $L_K$ denoted by $L^r_K$ is defined by
$$L^r_K(\sum_{k\geq 1} c_k \phi_k)=\sum_{k\geq 1} c_k \sigma^r_k \phi_k,$$ which is itself a positive compact operator on $L^2_{\rho_X}$.
We first introduce a regularity (often interpreted as smoothness) assumption on the regression function $f_{\rho}$.
\begin{assumption}\label{assumption1}
\begin{equation}\label{regularitycondition}
f_\rho= L_K^r u_\rho \quad \mbox{for $r>0$ and $u_\rho\in L_{\rho_X}^2$}.
\end{equation}
\end{assumption} This assumption implies that $f_{\rho}$ belongs to the range space of $L^r_K$ expressed as
$$L^r_{K}(L^2_{\rho_X})=\left\{f \in L^2_{\rho_X}: \sum_{k \geq 1} \frac{\langle f, \phi_k\rangle^2_{\rho}}{\sigma_k^{2r}}<\infty \right\}.$$ Then $L^{r_1}_{K}(L^2_{\rho_X}) \subseteq L^{r_2}_{K}( L^2_{\rho_X})$ whenever $r_1 \geq r_2$. The regularity of $f_{\rho}$ is measured by the decay rate of its expansion coefficients in terms of $\{\phi_k\}_{k \in \mathbb{N}}$. Condition (\ref{regularitycondition}) means that $\langle f, \phi_k\rangle^2_{\rho_X}$ decays faster than the $2r-$th power of the eigenvalues of $L_K$. Apparently, larger parameters $r$ will results in faster decay rates, and thus indicate higher regularities of $f_{\rho}$. This assumption is standard in the literature of learning theory, and can be further interpreted by the theory of interpolation spaces \cite{Smale2003}. Moreover, since $\rho_X$ is non-degenerate, from Theorem 4.12 in \cite{cucker2007learning}, $L^{1/2}_K$ is an isomorphism from $\overline{{\cal H}_K}$, the closure of ${\cal H}_K$ in $L^2_{\rho_X}$, to ${\cal H}_K$, i.e., for each $f\in \overline{{\cal H}_K}$, $L^{1/2}_K f \in {\cal H}_K$ and
\begin{equation}\label{normrelation2}
\|f\|_{\rho}=\|L^{1/2}_K f\|_K.
\end{equation} Therefore, $L^{1/2}_{K}(L^2_{\rho_X})={\cal H}_{K}$, and when $r>\frac12$, condition (\ref{regularitycondition}) implies $f_\rho\in\H_K $. When Assumption \ref{assumption1} holds with $0<r\le\frac12,$ convergence analysis for algorithm (\ref{algorithm}) with polynomially decaying step sizes has been established in \cite{Ying2008}. In this work, we focus on deriving fast and strong convergence rates for $r> \frac12$.

Let ${\rm Tr}(A)$ denote the trace of an operator $A$ of trace class. We also require the following capacity assumption based on the integral operator $L_K$.
\begin{assumption}\label{assumption2}
\mbox{${\rm Tr}(L_K^\beta)<\infty$ with $0\leq \beta< 1.$}
\end{assumption}
This assumption  arises naturally and plays a crucial role in our error analysis. Denote $\kappa=\sup_{x\in X} \sqrt{ K(x,x)}.$ Then $\kappa$ is a finite constant as $K$ is continuous and $X$ is compact. Due to the Mercer Theorem (e.g., see Theorem 4.10 in \cite{cucker2007learning}), $L_K$ is a trace class operator satisfying
$${\rm Tr}(L_K)=\sum_{k\geq 1} \sigma_k = \int_X K(x,x)d\rho_X \leq \kappa^2.$$ Hence Assumption \ref{assumption2} holds trivially with $\beta=1$. The definition of $L_K^\beta$ gives that ${\rm Tr}(L_K^\beta)=\sum_{k\ge 1}\sigma_k^\beta$. Assumption \ref{assumption2} is essentially an eigenvalue decaying condition imposed on the operator $L_K$, which is equivalent to the classical covering, or entropy number assumptions on the capacity of the RKHS ${\cal H}_K$ \cite{cucker2007learning,Steinwart2008}. In fact, if ${\rm Tr}(L_K^\beta)<\infty,$ since the eigenvalues $\{\sigma_k\}_{k\in \mathbb{N}}$ are sorted in a decreasing order, then for any $k\ge 1,$ we have
$$k\sigma_k^\beta\le \sum_{j=1}^k\sigma_j^\beta\le \sum_{j\ge 1}\sigma_j^\beta={\rm Tr}(L_K^\beta)<\infty.$$ It follows that $\sigma_k\le k^{-1/\beta}({\rm Tr}(L_K^\beta))^{1/\beta}, \forall k \geq 1$. A small value of $\beta$ implies a fast polynomially decaying rate at least achieved by the eigenvalues $\{\sigma_k\}_{k\in \mathbb{N}}$. If the eigenvalues have exponential decay, the index $\beta$ can be arbitrarily closed to zero. When the integral operator $L_K$ has finite rank, i.e., $\sigma_k=0$ for $k$ sufficiently large, one can choose $\beta=0$. Assumption \ref{assumption2} incorporates the smoothness of the kernel, as well as the information of the marginal distribution $\rho_X$, which can help us to establish sharp error estimates.

Throughout the paper, we suppose that the output $y$ is uniformly bounded, i.e., for some constant $M>0$, $|y|\le M$ almost surely. Recall that the sample $\{z_t=(x_t,y_t)\}_{t\in\N}$ is drawn independently from the probability distribution $\rho$. For $k\in\N,$ let $\bE_{z_1,\cdots,z_k}$ denote taking expectation with respect to $z_1,\cdots,z_k$, which is written as $\bE_{Z^k}$ for short. Our first main result establishes upper bounds for the convergence in $L^2_{\rho_X}$.

\begin{theorem}\label{theorem: rate in L2}
Define $\{f_t\}_{t\in \mathbb{N}}$ by (\ref{algorithm}). Suppose that Assumption \ref{assumption1} holds with $r>\frac12$ and Assumption \ref{assumption2} is satisfied with $0\leq \beta<1$. Then
\begin{enumerate}
\item[(a)] if $\frac12< r\le 1-\frac{\beta}{2},$ and $\eta_t=\eta_1 t^{-\frac{2r}{2r+1}}$ with  $0<\eta_1<\kappa^2,$ we have
\begin{equation*}
\bE_{Z^t}[\|f_{t+1}-f_\rho\|_\rho^2]\leq C_{\theta,r,\beta} t^{-\frac{2r}{2r+1}} ,
\end{equation*}
\item[(b)]  if $r>1-\frac{\beta}{2},$ and $\eta_t=\eta_1 t^{-\frac{2-\beta}{3-\beta}}$ with  $0<\eta_1<\kappa^2,$ we have
\begin{equation*}
\bE_{Z^t}[\|f_{t+1}-f_\rho\|_\rho^2]\leq C_{\theta,r,\beta} t^{-\frac{2-\beta}{3-\beta}}.
\end{equation*}
\end{enumerate} Here the constant $C_{\theta,r,\beta}$ is independent of $t$ and will be given in the proof.
\end{theorem}
Without requiring any complexity measures of the underlying function space, a capacity independent analysis for algorithm (\ref{algorithm}) with polynomially decaying step sizes is established in \cite{Ying2008}. The derived convergence rates cease improving when $r>\frac12$ and can not be faster than $\mathcal{O}(t^{-\frac12}\ln t)$. How to get better rates beyond the range $r\in (0,\frac12]$ is stated as an open problem in \cite{Ying2008}. Theorem \ref{theorem: rate in L2} shows that under a mild capacity assumption, i.e., Assumption \ref{assumption2} with $\beta <1$, one can obtain convergence rates faster than $\mathcal{O}(t^{-\frac12})$ when $r>\frac12$. As we discussed, this capacity assumption can be satisfied by various types of Mercer kernel $K$ with polynomial eigendecay, including the spline kernels and smooth Sobelev kernels \cite{wahba1990,Birman1967,Gu2002}. Thus the derived error bounds are valid for most Mercer kernels that are of interest in practice. For some analytic kernels on domains in Euclidean spaces or Gaussian-type kernels, the eigenvalues of $L_K$ have an exponential decay. Then $\beta$ can be arbitrarily small in Assumption \ref{assumption2} and the upper bound of convergence rates can almost achieve to $\mathcal{O}(t^{-\frac23})$ when $r>1$. To our best knowledge, Theorem \ref{theorem: rate in L2} gives the first capacity dependent rates for the last iteration of algorithm (\ref{algorithm}) with polynomially decaying step sizes, which greatly improves the existing results. For detailed comparisons, one can refer to Section \ref{section: related work}. However, the rates derived in Theorem \ref{theorem: rate in L2} still suffer a saturation problem, that is, when $r>1-\frac{\beta}{2},$ the rates will not be improved. The same problem is also observed in \cite{Dieuleveut&Bach2016}, when establishing capacity dependent rates for the averaging scheme of algorithm (\ref{algorithm}), which we will briefly discuss in Section \ref{section: related work}.

Our second main result gives the upper bounds for the strong convergence in $\H_K$.
\begin{theorem}\label{theorem: rate in HK}
Define $\{f_t\}_{t \in \mathbb{N}}$ by (\ref{algorithm}). Suppose that Assumption \ref{assumption1} holds with $r> \frac12$ and Assumption \ref{assumption2} is satisfied with $0<\beta<1$. Let $\eta_t=\eta_1 t^{-\frac12}$ with $0<\eta_1<\frac1{\kappa^2},$ then
\begin{eqnarray*}
\bE_{Z^t}[\|f_{t+1}-f_\rho\|_K^2]\le \widetilde{C}_{\theta,r,\beta}  t^{-\frac{\min\{2r-1,1-\beta\}}{2}}(\ln t)^2,
\end{eqnarray*} where the constant $\widetilde{C}_{\theta,r,\beta}$ is independent of $t$ and will be given in the proof.
\end{theorem}

It should be pointed out that, for algorithm (\ref{algorithm}) with polynomially decaying step sizes, Theorem \ref{theorem: rate in HK} establishes the first convergence result in ${\cal H}_K$. One can see that if the eigenvalues of $L_K$ decay exponentially, then $\beta$ can be arbitrarily small and the convergence rate can almost achieve to $\mathcal{O}(t^{-\frac{1}{2}}(\ln t)^2)$ when $r\ge1.$ Our main results demonstrate that the theoretical analysis developed in this paper fully exploits the fine structure of the kernel function $K$ and the RKHS ${\cal H}_K$, and thus can lead to sharp error estimates of online learning algorithm. We will prove all these results in Section \ref{section: convergenceinL2} and Section \ref{section: covergenceinHK}.

\section{Related Work}\label{section: related work}
In this section we compare our error analysis for algorithm (\ref{algorithm}) with some existing results in
the literature.

Algorithm (\ref{algorithm}) was thoroughly investigated in \cite{Ying2008} via a capacity independent approach, the performance of the last iterate
with polynomially decaying step sizes and constant step sizes (depends on the total number of the iterates, e.g., the sample size) was studied. Though better convergence rates can be provided in the case of constant step sizes, as we discussed before, algorithm (\ref{algorithm}) with constant step sizes is not a truly online learning algorithm since we need a prior knowledge of the sample size. We compare our result with that in \cite{Ying2008} for the case of polynomially decaying step sizes. It shows in \cite{Ying2008} that if Assumption \ref{assumption1} is satisfied with $0<r\le\frac12$, and $\eta_t=\left(1/\mu(2r/2r+1)\right)t^{-\frac{2r}{2r+1}}$ with some $\mu>0$, there holds
\begin{eqnarray*}
\bE_{Z^t}[\|f_{t+1}-f_\rho\|_\rho^2] =\mathcal{O}(t^{-\frac{2r}{2r+1}}\ln t).
\end{eqnarray*}
One can easily see that the best convergence rates is $\mathcal{O}(t^{-\frac12}\ln t)$ achieved at $r=\frac12.$ And the analysis established in \cite{Ying2008} can not lead to rates faster than $\mathcal{O}(t^{-\frac12}\ln t)$ when $f_\rho$ has higher regularities, i.e., $r\ge \frac12$ in Assumption \ref{assumption1}. The error analysis is capacity independent, that is, asides from the regularity conditions for $f_\rho$ stated in Assumption \ref{assumption1}, there is no assumptions imposed on the capacity of the underlying RKHS $\H_K$. Whether better rates can be derived with some additional capacity information is an open problem stated in \cite{Ying2008}. One can see that if assumption 1 holds with $r>\frac12$ and assumption 2 holds with $0<\beta<1,$ our  convergence rates in Theorem \ref{theorem: rate in L2} is always faster than $\mathcal{O}(t^{-\frac12}),$ thereby we give a positive answer to the question proposed in \cite{Ying2008}.

A recent work \cite{Dieuleveut&Bach2016} provides capacity dependent estimates for algorithm (\ref{algorithm}) depending on a polynomial eigendecay condition of $L_K.$ Instead of considering the last iterate, the authors investigate the averaging estimator of each iterates, i.e., $\bar{f}_t=\frac{1}{t}\sum_{i=1}^t f_i,$ where $\{f_i\}_{i \in \mathbb{N}}$ is given by algorithm (\ref{algorithm}). When Assumption \ref{assumption1} holds with a regularity parameter $r>0$, it shows that the mini-max optimal rate can be derived for some region of $r$, if the eigenvalues of $L_K$ behave as $c_1 k^{-\frac{1}{\beta}}\le \sigma_k\le c_2 k^{-\frac{1}{\beta}}$ with $0<\beta<1$ and constants $c_1, c_2>0.$ Let $f_\H$ denote the orthogonal projection of $f_\rho$ on ${\cal H}_K$. Particularly, $f_H=f_{\rho}$ if $r\geq \frac12$. The analysis in \cite{Dieuleveut&Bach2016} shows that
\begin{enumerate}
         \item  if $\frac{1}{2}-\frac{\beta}{2}<r\le1-\frac{\beta}{2},$ with $\eta_i=\eta_1 i^{\frac{-2r-\beta+1}{2r+\beta}},$
\begin{eqnarray*}\label{averaging rate in L2}
\bE_{Z^{t-1}}\|\bar{f}_t-f_\H\|_\rho^2=\mathcal{O}(t^{-\frac{2r}{2r+\beta}});
\end{eqnarray*}
         \item   if $r>1-\frac{\beta}{2},$ with $\eta_i=\eta_1 i^{-\frac12},$
\begin{eqnarray*}\label{averaging rate in HK}
\bE_{Z^{t-1}}\|\bar{f}_t-f_\H\|_\rho^2=\mathcal{O}(t^{-(1-\frac{\beta}{2})}).
\end{eqnarray*}
\end{enumerate}
One can see that the obtained asymptotic convergence rate is optimal when $\frac{1}{2}-\frac{\beta}{2}<r<1-\frac{\beta}{2}$, which achieves the mini-max lower bound proved in \cite{Caponnetto2007,Steinwart2009}. Both our results and the above results in \cite{Dieuleveut&Bach2016} are capacity dependent. As we discussed, the capacity condition in Assumption \ref{assumption2} is more general than the polynomial eigendecay condition adopted in \cite{Dieuleveut&Bach2016} and we do not require the lower bound for the decaying rates. We see that the convergence rates in Theorem \ref{theorem: rate in L2} is slightly worse than that of \cite{Dieuleveut&Bach2016}. This is mainly because that averaging scheme can reduce variance and thus result in robust estimators and better convergence rates (e.g., see \cite{rakhlin2012making,nemirovski2009robust,yao2010complexity}). However, the time complexity of taking average to generate the estimator is $\mathcal{O}(T^2)$ when algorithm (\ref{algorithm}) stops updating after $T-$th iteration. Such scaling of complexity is still prohibitive for large data sets.

It should be pointed out that the established estimates for algorithm (\ref{algorithm}) with polynomially decaying step sizes, whether it is capacity independent or capacity dependent, all suffer a saturation problem, i.e., the convergence rate on longer improves once the regularity of the regression function is beyond certain level. Concretely, if regularity condition in Assumption \ref{assumption1} is satisfied with $r>0$ and Assumption \ref{assumption2} holds with $0<\beta<1$, the capacity independent rate in \cite{Ying2008} is saturate at $r=\frac12$ and ceases to improve as $r>\frac12$, while the capacity dependent rate in \cite{Dieuleveut&Bach2016} and our paper is saturate at $r=1-\frac{\beta}{2}$ and stops getting better when $r>1-\frac{\beta}{2}$. However, the gradient descent in batch learning, known as $L_2$ boosting in machine learning and Landweber iterations in statistical inverse problem \cite{zhang2005boosting,Yao2007}, is unsaturate and adapts to favorable regularity of $f_{\rho}$ to attain even faster convergence rates. As verified by theoretical and empirical analysis, shrinking the step size $\{\eta_t\}_{t\in \mathbb{N}}$ plays different roles in batch learning and online learning. In batch learning, it only affects the stopping time of the iteration and does not contribute to the regularization.  However, in online learning, it significantly affects the regularization and convergence rates. The capacity dependent analysis of online learning is much involved than that of batch learning. To derive fast convergence rate, one need to choose the step size carefully to settle a bias-variance trade-off based on regularity and capacity information. Back to our results, notice that Assumption \ref{assumption2} holds trivially with $\beta=1$. Taking $\beta=1$ with a mild modification in our error analysis will lead to the best capacity independent rate $\mathcal{O}(t^{-\frac12}\ln t)$ obtained in \cite{Ying2008} which is saturated at $r=\frac12$. Hence the capacity condition in Assumption \ref{assumption2} plays a key role in improving the saturation effect suffered by capacity independent analysis in \cite{Ying2008}. But the capacity information can not affect the derived rates when $\frac12<r\leq 1-\frac{\beta}{2}$, it only improves the constant estimates. That is, a smaller capacity parameter $\beta$ will lead to a tighter estimate for $C_{\theta,r,\beta}$. The capacity information is reflected in the convergence rate only after being saturate. It is still not clear that whether the saturation problem comes from the regularization mechanism of algorithm (\ref{algorithm}), or comes from the limitation of the analysis approach. An elaborated mini-max analysis specific to online learning algorithm can help us to clarify this problem and verify the optimality of the derived rates, which will be our future research topic. If one can obtain the prior information on the sample size $T$, the saturation problem can be eliminated by choosing a constant step size $\eta_{t}=\eta(T)$ depending on $T$. In \cite{Ying2008},
 if Assumption \ref{assumption1} is satisfied with $r>0$, let $\eta_t=[r/64(1+\kappa)^4(2r+1)]T^{-\frac{2r}{2r+1}}$, the capacity independent analysis gives that
 \begin{eqnarray*}
\bE_{Z^T}[\|f_{T+1}-f_\rho\|_\rho^2] =\mathcal{O}(T^{-\frac{2r}{2r+1}}\ln T).
\end{eqnarray*} Using the error analysis framework established in this paper, when Assumption \ref{assumption2} is satisfied with $\beta<1$, by choosing $\eta_t=[1/\kappa] T^{-\frac{2r}{2r+1}}$, we can remove the logarithmic term in the above estimate, i.e.,
 \begin{eqnarray*}
\bE_{Z^T}[\|f_{T+1}-f_\rho\|_\rho^2] =\mathcal{O}(T^{-\frac{2r}{2r+1}}).
\end{eqnarray*} While $\beta=1$ and $\eta_t=[1/\kappa] T^{-\frac{2r}{2r+1}}$, we obtain the same capacity independent rate in \cite{Ying2008}. The derived convergence rates are comparable with that of batch learning and can be arbitrarily closed to $T^{-1}$ as $r\to \infty$. Through the above comparisons, we see that, the capacity information helps to refine the previous convergence analysis of algorithm (\ref{algorithm}).

Besides unregularized online learning algorithm considered in this paper, the regularized version has also received considerable attention \cite{ying2006online,Smale&Yao2006,Kivinen2004,tarres2014online}. Regularized online learning can be viewed as the stochastic gradient descent method to solve the Tikhonov regularized scheme in an RKHS. In particular, the regularized least-squares online learning algorithm with time-varying regularization parameters $\{\lambda_t\}_{t\in \mathbb{N}}$ updates iterates as follows
\begin{equation}\label{regularizedonline}
  f_{t}=f_{t-1}-\eta_t[(f_{t-1}(x_t)-y_t)K_{x_t}+\lambda_tf_{t-1}], \quad \forall t\in\N,
\end{equation}
where $f_0\in\H_K$ is fixed, e.g., $f_0=0.$ For each sample $z_t=(x_t,y_t)$, the stochastic gradient is given by $2[(f_{t-1}(x_t)-y_t)K_{x_t}+\lambda_tf_{t-1}]$, which approximates the gradient of $\E(f)+\lambda\|f\|_K^2$ at $f=f_{t-1}$ and $\lambda=\lambda_t$. It shows in \cite{tarres2014online} that let $\eta_t=at^{-\frac{2r}{2r+1}}$ and $\lambda_t=\frac{1}{a}t^{-\frac{1}{2r+1}}$ for some constant $a>1,$
\begin{enumerate}
\item if Assumption \ref{assumption1} holds with $\frac12< r\le 1,$  then for $0<\delta<1,$ with probability at least $1-\delta,$
\begin{equation*}\label{regularized rate in L2}
\|f_t-f_\rho\|_\rho^2=\mathcal{O}(t^{-\frac{2r}{2r+1}}(\log\delta/2)^4), \quad\forall t\in\N;
\end{equation*}
\item if Assumption \ref{assumption1} holds with $\frac12< r\le \frac32,$  then for $0<\delta<1,$ with probability at least $1-\delta,$
\begin{equation*}\label{regularized rate in HK}
\|f_t-f_\rho\|_K^2=\mathcal{O}(t^{-\frac{2r-1}{2r+1}}(\log\delta/2)^2), \quad\forall t\in\N.
\end{equation*}
\end{enumerate}
Due to the additional Tikhonov regularization term, the regularized online learning algorithm (\ref{regularizedonline}) is a stochastic approximation of gradient descent method to minimize a strong convex objective functional in an RKHS. Thanks to the strong convexity, its convergence rate is better than that of online learning algorithm without regularization, but still suffers the saturation effect incurred by Tikhonov regularization. By introducing the capacity information in Assumption \ref{assumption2}, for smooth Mercer Kernels with an arbitrarily small $\beta$, e.g, Gaussian kernels and polynomial kernels, our results in Theorem \ref{theorem: rate in L2} and in Theorem \ref{theorem: rate in HK} are comparable with those derived in \cite{tarres2014online}. To be specific, when $\frac12<r\le 1-\frac{\beta}{2}$, our convergence rate in $L^2_{\rho_X}$ is the same as that of algorithm (\ref{regularizedonline}). And for the convergence in $\H_K,$ since $\min\{r-\frac12,\frac12-\frac\beta2\}>\frac{2r-1}{2r+1}$, our convergence rate is slight better than that of algorithm (\ref{regularizedonline}). Though we only get the convergence in expectation, one might attempt to obtain probabilistic capacity dependent bound using the idea in the recent work \cite{Lei2017}. The main drawback of the regularization scheme \eqref{regularizedonline} is that it require tuning two sequences of hyperparameters (the step-size in stochastic gradient descent and a regularization parameter), which is time consuming in practice. Therefore, in real data analysis, one prefers to the use unregularized online leaning scheme instead of its regularized counterpart. Our error analysis guarantees that, in most situations, the performances of these two algorithm are comparable.

\section{Convergence Analysis in $L_{\rho_X}^2$}\label{section: convergenceinL2}
In this section, we prove the convergence result in  $L_{\rho_X}^2$. Firstly, we establish the following error decomposition.
\begin{proposition}\label{prposition: error decomposition in L2}
Let $\{f_t\}_{t\in\N}$ be defined by (\ref{algorithm}). Then there holds
\begin{equation}\label{equ: error decomposition}
\begin{split}
\bE_{Z^t}\|f_{t+1}-f_\rho\|_\rho^2&\le \left\|\prod_{i=1}^t(I-\eta_t L_K)f_\rho\right\|_\rho^2\\
&\quad +2\sum_{i=1}^t \eta_i^2 \left(\kappa^2\bE_{Z^i}[\|f_i\|_K^2]+M^2\right){\rm Tr}\left(L_K^2 \prod_{j=i+1}^t(I-\eta_j L_K)^2 \right).
\end{split}
\end{equation}
\end{proposition}
Before proving the above proposition, we give some notations and useful lemmas. Let $({\cal H},\langle\cdot,\cdot\rangle_{{\cal H}})$ be a Hilbert space and $A: {\cal H} \to {\cal H}$ is a linear operator. The operator norm is defined by $\|A\|_{{\cal H}}=\sup_{\|f\|_{{\cal H}}=1}\|Af\|_{{\cal H}}$. When ${\cal H}$ is clear from the context, we will omit the subscript and simply denote the operator norm as $\|\cdot\|$. Then $\|Af\|_{\cal H}\leq \|A\| \|f\|_{\cal H}, \forall f \in {\cal H}$. Let $A^T$ be the adjoint operator of $A$ such that $\langle Af,f'\rangle_{\cal H} = \langle f,A^Tf'\rangle_{\cal H}, \forall f,f' \in {\cal H}$. We say that $A$ is self-adjoint if $A^T=A$, and positive, i.e., $A\succeq 0,$ if $A$ is self-adjoint and $\langle Af, f\rangle_{{\cal H}} \geq 0$ for all $f\in {\cal H}$. And $A\succeq B$ (or $B\preceq A$) means $A,B$ are self-adjoint on ${\cal H}$ and $A-B\succeq 0,$ or equivalently, $\langle Af, f\rangle_{{\cal H}} \geq \langle Bf, f\rangle_{{\cal H}}, \forall f \in {\cal H}$. For $f,f'\in {\cal H}$, define a rank one operator $f \otimes f': {\cal H} \to {\cal H}$ by $f \otimes f' (h)= \langle f', h \rangle_{\cal H} f, \forall h \in {\cal H}$. One can check that $$L_K=\int_X \langle K_x, \cdot \rangle_K K_x d\rho_X=\bE\left[K_{x} \otimes K_{x}\right],$$ where the expectation is formally defined as a Bochner integral in the set of linear bounded operators on ${\cal H}_K$ \cite{Miku1978}. In the sequel, for $ \ell \leq k$, denote by
$$\bE_{z_\ell|Z^{k}}=\bE[\cdot|z_1,\cdots,z_{\ell-1},z_{\ell+1},\cdots,z_k]$$
the conditional expectation on $z_\ell$ with respect to $z_1,\cdots,z_{\ell-1},z_{\ell+1},\cdots,z_k$.

\begin{lemma}\label{lemma: key lemma}
 Let $\{f_i\}_{i\in\N}$ be defined as (\ref{algorithm}). Then there holds
\begin{eqnarray*}
\bE_{Z^i}\left[(f_i(x_i)-y_i)^2K_{x_i} \otimes K_{x_i} \right]\preceq 2\left(\kappa^2\bE_{Z^{i-1}}[\|f_i\|_K^2]+M^2\right)L_K.
\end{eqnarray*}
\end{lemma}
\begin{proof}
The reproducing property (\ref{reproducingproperty}) combing with Cauchy-Schwarz inequality yields that
\begin{equation*}\label{normrelation1}
\sup_{x\in X} |f(x)| \leq \kappa \|f\|_K, \quad \forall f \in {\cal H}_K.
\end{equation*}
The inequality above together with the boundedness of $y_i$ implies
\begin{eqnarray*}
(f_i(x_i)-y_i)^2\le 2(\kappa^2\|f_i\|_K^2+M^2).
\end{eqnarray*}
For $i\in\N,$ since $f_i$ depends only on $z_1, \cdots, z_{i-1},$ then for any $g\in \H_K,$
\begin{equation*}
\begin{split}
\langle \bE_{z_i|Z^{i}}\left[(f_i(x_i)-y_i)^2K_{x_i} \otimes K_{x_i}\right]g,g \rangle_K &\le 2(\kappa^2\|f_i\|_K^2+M^2) \langle \bE_{z_i|Z^{i}}[K_{x_i} \otimes K_{x_i}]g,g \rangle_K\\
&=2(\kappa^2\|f_i\|_K^2+M^2) \langle L_K g, g\rangle_K.
\end{split}
\end{equation*}It follows that
\begin{eqnarray*}
\bE_{z_i|Z^{i}}[(f_i(x_i)-y_i)^2K_{x_i} \otimes K_{x_i}]\preceq 2(\kappa^2\|f_i\|_K^2+M^2)L_K
\end{eqnarray*}
and therefore,
\begin{equation*}
\begin{split}
\bE_{Z^i}\left[(f_i(x_i)-y_i)^2K_{x_i} \otimes K_{x_i} \right]&=\bE_{Z^{i-1}}\bE_{z_i|Z^{i}}\left[(f_i(x_i)-y_i)^2K_{x_i} \otimes K_{x_i} \right]\\
&\preceq  2(\kappa^2\bE_{Z^{i-1}}[\|f_i\|_K^2]+M^2)L_K.
\end{split}
\end{equation*}This completes the proof.
\end{proof}

\begin{lemma}\label{lemma4.2}
Let $A$ and $B$ be self-adjoint operators on a Hilbert space ${\cal H}$. If $A\succeq B,$ then for any operator $C$ on ${\cal H},$ $C^T AC\succeq C^TBC.$
\end{lemma}
\begin{proof}
Since $A\succeq B,$ i.e., $A-B\succeq 0,$ which implies $\langle (A-B)f, f\rangle_{\cal H}\ge0, \forall f\in {\cal H}.$  Then for any operator $C$ on ${\cal H},$ there holds $\langle  C^T(A-B)Cf,f\rangle_{\cal H}=\langle (A-B)Cf,Cf\rangle_{\cal H}\ge 0.$ We thus complete the proof of the lemma.
\end{proof}

Applying lemma \ref{lemma4.2} to $A=\bE_{Z^i}\left[(f_i(x_i)-y_i)^2K_{x_i} \otimes K_{x_i} \right],$ $B=2\left(\kappa^2\bE_{Z^{i-1}}[\|f_i\|_K^2]+M^2\right)L_K,$ and $C=L_K^{1/2}\prod_{j=i+1}^t(I-\eta_j L_K)$ yields that
\begin{eqnarray}\label{equ: key inequality}
&&L_K^{\frac12} \prod_{j=i+1}^t(I-\eta_j L_K) \bE_{Z^i}\left[(f_i(x_i)-y_i)^2(K_{x_i}\otimes K_{x_i}) \right]\prod_{j=i+1}^t(I-\eta_j L_K)L_K^{\frac12} \nonumber\\
&&\preceq 2\left(\kappa^2\bE_{Z^{i-1}}[\|f_i\|_K^2]+M^2\right)\left(L_K^{\frac12} \prod_{j=i+1}^t(I-\eta_j L_K)L_K \prod_{j=i+1}^t(I-\eta_j L_K)L_K^{\frac12}\right)\\
&&= 2\left(\kappa^2\bE_{Z^{i-1}}[\|f_i\|_K^2]+M^2\right)\left(L_K^{2} \prod_{j=i+1}^t(I-\eta_j L_K)^2\right).\nonumber
\end{eqnarray} Now we are ready to prove Proposition \ref{prposition: error decomposition in L2}.

\noindent{\bf Proof of Proposition \ref{prposition: error decomposition in L2}}.
Define $\prod_{j=t+1}^t(I-\eta_j L_K)=I$. By the definition of $f_t$ given by (\ref{algorithm}), we have
\begin{equation*}\label{decomposition1}
\begin{split}
f_{t+1}-f_\rho&
=f_t-\eta_t(f_t(x_t)-y_t)K_{x_t}-f_\rho\\
&=(I-\eta_t L_K)(f_t-f_\rho)+\eta_t \mathcal{B}_t\\
&=-\prod_{i=1}^t(I-\eta_t L_K)f_\rho+\sum_{i=1}^t\prod_{j=i+1}^t(I-\eta_j L_K)\eta_i \mathcal{B}_i,
\end{split}
\end{equation*}
where $\mathcal{B}_i=(y_i-f_i(x_i))K_{x_i}-L_K(f_\rho-f_i),\forall 1\leq i \leq t.$ Since $f_i$ depends only on $z_1,\cdots,z_{i-1},$ we have $\bE_{z_i|Z^{i}}[\mathcal{B}_i]=0.$ Then it follows that
\begin{eqnarray}\label{equ: first step}
&&\bE_{Z^t}\|f_{t+1}-f_\rho\|_\rho^2=\bE_{Z^t}\left[\left\|-\prod_{i=1}^t(I-\eta_t L_K)f_\rho+\sum_{i=1}^t\prod_{j=i+1}^t(I-\eta_j L_K)\eta_i \mathcal{B}_i\right\|_\rho^2\right] \nonumber\\
&&=\left\|\prod_{i=1}^t(I-\eta_t L_K)f_\rho\right\|_\rho^2+\bE_{Z^t}\left[\left\|\sum_{i=1}^t\prod_{j=i+1}^t(I-\eta_j L_K)\eta_i \mathcal{B}_i\right\|_\rho^2\right]\\ \nonumber
&&-2\bE_{Z^t}\left[\langle\prod_{i=1}^t(I-\eta_i L_K)f_\rho,\sum_{i=1}^t\prod_{j=i+1}^t(1-\eta_j L_K)\eta_i \mathcal{B}_i\rangle_\rho\right]\\
&&=\left\|\prod_{i=1}^t(I-\eta_t L_K)f_\rho\right\|_\rho^2+ \bE_{Z^t}\left[\left\|\sum_{i=1}^t\prod_{j=i+1}^t(I-\eta_j L_K)\eta_i \mathcal{B}_i\right\|_\rho^2\right].\nonumber
\end{eqnarray}
For the second term on the right hand side of the above equation, we have
\begin{eqnarray*}
&&\bE_{Z^t}\left[\left\|\sum_{i=1}^t\prod_{j=i+1}^t(I-\eta_j L_K)\eta_i \mathcal{B}_i\right\|_\rho^2\right]=\sum_{i=1}^t\eta_i^2 \bE_{Z^t}\left[\left\|\prod_{j=i+1}^t(I-\eta_j L_K) \mathcal{B}_i\right\|_\rho^2\right]\\
&&+\sum_{i=1}^t\sum_{k\neq i}\eta_i\eta_k \bE_{Z^t}\langle\prod_{j=i+1}^t(I-\eta_j L_K) \mathcal{B}_i, \prod_{j'=k+1}^t(I-\eta_j' L_K) \mathcal{B}_k\rangle_\rho\\
&&=\sum_{i=1}^t\eta_i^2 \bE_{Z^t}\left[\left\|\prod_{j=i+1}^t(I-\eta_j L_K) \mathcal{B}_i\right\|_\rho^2\right].
\end{eqnarray*}
The last equality holds since for $i>k,$ $$\bE_{z_i|Z^{t}}\left[\langle\prod_{j=i+1}^t(I-\eta_j L_K) \mathcal{B}_i, \prod_{j'=k+1}^t(I-\eta_j' L_K) \mathcal{B}_k\rangle_\rho\right]=0$$ and also for $i<k,$ $$\bE_{z_k|Z^{t}}\left[\langle\prod_{j=i+1}^t(I-\eta_j L_K) \mathcal{B}_i, \prod_{j'=k+1}^t(I-\eta_j' L_K) \mathcal{B}_k\rangle_\rho\right]=0.$$
Now we see from the identity (\ref{normrelation2}) for $f\in\H_K$ and inequality (\ref{equ: key inequality}) that
\begin{eqnarray*}
&&\sum_{i=1}^t \eta_i^2\bE_{Z^t}\left[\left\|\prod_{j=i+1}^t(I-\eta_j L_K) \mathcal{B}_i\right\|_\rho^2\right]
=\sum_{i=1}^t\eta_i^2 \bE_{Z^t}\left[\left\|L_K^{1/2}\prod_{j=i+1}^t(I-\eta_j L_K) \mathcal{B}_i\right\|_K^2\right]\\
&&\le\sum_{i=1}^t\eta_i^2 \bE_{Z^t}\left[\left\|L_K^{1/2}\prod_{j=i+1}^t(I-\eta_j L_K) (f_i(x_i)-y_i)K_{x_i}\right\|_K^2\right]\\
&&=\sum_{i=1}^t \eta_i^2  \bE_{Z^t}\left[{\rm Tr}\left(L_K^{1/2}\prod_{j=i+1}^t(I-\eta_j L_K) ((f_i(x_i)-y_i)K_{x_i})\otimes ((f_i(x_i)-y_i)K_{x_i}) \prod_{j=i+1}^t(I-\eta_j L_K)L_K^{1/2}\right)\right]\\
&&=\sum_{i=1}^t \eta_i^2 \left[{\rm Tr}\left(L_K^{1/2}\prod_{j=i+1}^t(I-\eta_j L_K) \bE_{Z^i}\left[((f_i(x_i)-y_i)K_{x_i})\otimes ((f_i(x_i)-y_i)K_{x_i})\right] \prod_{j=i+1}^t(I-\eta_j L_K)L_K^{1/2}\right)\right]\\
&&\le \sum_{i=1}^t \eta_i^2 \left(2\kappa^2\bE_{Z^{i-1}}[\|f_i\|_K^2]+2M^2\right) {\rm Tr}\left(L_K^2 \prod_{j=i+1}^t(I-\eta_j L_K)^2 \right),
\end{eqnarray*} where the second equality is due to the fact that $\|f\|^2_K={\rm Tr}(f\otimes f), \forall f\in {\cal H}_K$. The proof is then completed by putting the above estimates into (\ref{equ: first step}).
\qed

In the following, we estimate the two terms in Proposition \ref{prposition: error decomposition in L2} respectively. We first prove some elementary inequalities.
\begin{lemma}\label{lemma: eta bound}
Let $\eta_i=\eta_1 i^{-\theta}$ with $0<\eta_1<\frac1{\kappa^2}$ and $\frac12\le\theta<1,$  and $t\ge 3,$ we have
\begin{enumerate}
\item[(a)] $\frac{\eta_1(1-2^{\theta-1})}{1-\theta}t^{1-\theta}\le \sum_{i=1}^t \eta_i\le \frac{\eta_1}{1-\theta}t^{1-\theta}.$
\item[(b)]
$\sum_{i=1}^t \eta_i^2 \le\left\{\begin{array}{ll} \frac{2\eta_1^2\theta}{2\theta-1}, & \hbox{for} \
\frac{1}{2}<\theta<1, \\ 2\eta_1^2\ln t, & \hbox{for} \ \theta=\frac12. \end{array}\right.$
\item[(c)]  $\sum_{j=i+1}^t\eta_j\ge \frac{\eta_1}{1-\theta}\left((t+1)^{1-\theta}-(i+1)^{1-\theta}\right).$
\end{enumerate}
\end{lemma}
\begin{proof}
For part (a), on one hand, we have $$\sum_{i=1}^t \eta_i=   \sum_{i=1}^t \eta_1 i^{-\theta}\le \eta_1\int_{0}^t x^{-\theta}dx= \frac{\eta_1}{1-\theta}t^{1-\theta}.$$ On the other hand,
\begin{equation*}
\begin{split}
&\sum_{i=1}^t \eta_i=\eta_1\sum_{i=1}^t i^{-\theta}\ge \eta_1\int_{1}^{t+1} x^{-\theta}dx=\frac{\eta_1}{1-\theta}\left((t+1)^{1-\theta}-1\right)\\
&\quad \ge \frac{\eta_1}{1-\theta}(t+1)^{1-\theta}\left(1-2^{\theta-1}\right)\ge \frac{\eta_1(1-2^{\theta-1})}{1-\theta}t^{1-\theta}.
\end{split}
\end{equation*}

For part (b), when $\frac12<\theta<1,$ we have
\begin{equation*}
\begin{split}
&\sum_{i=1}^t \eta_i^2=\sum_{i=1}^t \eta_1^2 i^{-2\theta}= \eta_1^2\left(1+\sum_{i=2}^t  i^{-2\theta}\right)\\
&\quad \le \eta_1^2\left(1+\int_{1}^t  x^{-2\theta}dx\right)\le \eta_1^2\left(1+\frac{1}{2\theta-1}\right) =\frac{2\eta_1^2\theta}{2\theta-1}.
\end{split}
\end{equation*}
And for $\theta=\frac12,$ one has
\begin{equation*}
\begin{split}
&\sum_{i=1}^t \eta_i^2=\sum_{i=1}^t \eta_1^2 i^{-1}= \eta_1^2\left(1+\sum_{i=2}^t  i^{-1}\right)\\
&\quad \le \eta_1^2\left(1+\int_{1}^t  x^{-1}dx\right) \le \eta_1^2\left(1+\ln t\right)\le 2\eta_1^2 \ln t.
\end{split}
\end{equation*}

Part (c) follows directly from $$\sum_{j=i+1}^t\eta_j=\eta_1\sum_{j=i+1}^tj^{-\theta}\ge \eta_1\int_{i+1}^{t+1}x^{-\theta}dx= \frac{\eta_1}{1-\theta}((t+1)^{1-\theta}-(i+1)^{1-\theta}).$$
\end{proof}

 For the first term of (\ref{equ: error decomposition}), we have the following bound.

\begin{proposition}\label{proposition: error 1}
Suppose $f_\rho$ satisfies Assumption \ref{assumption1} with $r> \frac12$ and $\eta_i=\eta_1 i^{-\theta}$ with $0<\eta_1<\frac{1}{\kappa^2}$ and $\frac12\le \theta<1$. Then
\begin{eqnarray*}
\left\|\prod_{i=1}^t(I-\eta_i L_K)f_\rho\right\|_\rho^2\le \left\|u_\rho\right\|_\rho^2 \left(\frac{r(1-\theta)}{e\eta_1(1-2^{\theta-1})}\right)^{2r} t^{-2r(1-\theta)},
\end{eqnarray*} where $e$ is the base of the natural logarithm.
\end{proposition}
\begin{proof} By the identity (\ref{normrelation2}) for $f\in \H_K$ and the assumption $f_\rho=L_K^ru_\rho$ with $u_\rho\in L_{\rho_X}^2$ and $r>\frac12$, we have
\begin{eqnarray*}
&&\left\|\prod_{i=1}^t(I-\eta_i L_K)f_\rho\right\|_\rho
=\left\|L_K^{1/2}\prod_{i=1}^t(I-\eta_i L_K)L_K^{r-1/2}L^{1/2}_Ku_\rho\right\|_K\\
&&\le\left\|\prod_{i=1}^t(I-\eta_i L_K)L_K^{r}\right\|\left\|u_\rho\right\|_\rho\le \sup_{x>0}\left(\prod_{i=1}^t(1-\eta_i x)x^r\right)\left\|u_\rho\right\|_\rho\le \left\|u_\rho\right\|_\rho\left(\frac{r}{e}\right)^{r}\left(\sum_{i=1}^t\eta_i\right)^{-r}.
\end{eqnarray*}
The last inequality holds due to the fact that $\sup_{x>0}\left(\prod_{i=1}^t(1-\eta_i x)x^r\right)\le \sup_{x>0}\left(\exp(-x\sum_{i=1}^t\eta_i )x^r\right)$ and the function $\exp(-x\sum_{i=1}^t\eta_i )x^r$ attains its maximum at $x=\frac{r}{\sum_{i=1}^t\eta_i}.$
Then by part (a) of Lemma \ref{lemma: eta bound}, we have
\begin{eqnarray*}
\left\|\prod_{i=1}^t(I-\eta_i L_K)f_\rho\right\|_\rho^2\le \left\|u_\rho\right\|_\rho^2 \left(\frac{r(1-\theta)}{e\eta_1(1-2^{\theta-1})}\right)^{2r} t^{-2r(1-\theta)}.
\end{eqnarray*} This completes the proof.
\end{proof}
Before estimating the second term (\ref{equ: error decomposition}) in Proposition \ref{prposition: error decomposition in L2}, we give the uniform bounds for the norm of iterates in expectation.
\begin{lemma}\label{lemma: ft bound}
Let $\{f_t\}_{t\in\N}$ be defined by (\ref{algorithm}), and $\eta_t=\eta_1t^{-\theta}$ with $0<\eta_1<\frac1{\kappa^2}$ and $\frac12\le\theta<1,$ there holds
\begin{eqnarray*}
\bE_{Z^t}[\|f_{t+1}\|_K^2]\le \left\{\begin{array}{ll} 4\|f_\rho\|_K^2+\frac{4\theta\eta_1^2\kappa^2}{2\theta-1}\left(20\|f_\rho\|_\rho^2+3\E(f_\rho)\right), & \hbox{for} \ \frac12<\theta<1, \\ \left(4\|f_\rho\|_K^2+4\eta_1^2\kappa^2\left(20\|f_\rho\|_\rho^2+3\E(f_\rho)\right)\right)\ln t, & \hbox{for} \ \theta=\frac12. \end{array}\right.
\end{eqnarray*}
\end{lemma}
\begin{proof} Recall that $\mathcal{B}_i= (y_i-f_i(x_i))K_{x_i}-L_K(f_\rho-f_i)$ for $1 \leq i \leq t$. Then $\bE_{z_i|Z^i}\left[\mathcal{B}_i\right]=0$ and $$\bE_{z_i|Z^i}\left[\|\mathcal{B}_i\|_K^2\right]\le\bE_{z_i|Z^i}\left[\|(y_i-f_i(x_i))K_{x_i}\|_K^2\right]=\bE_{z_i|Z^i}\left[(y_i-f_i(x_i))^2K(x_i,x_i)\right]\le \kappa^2\E(f_i),$$ where we still use the fact that $f_i$ is a random variable independent of $z_i$.  Combining with the definition of $f_t$ given by (\ref{algorithm}), we have
\begin{eqnarray*}
&&\bE_{Z^t}\left[\|f_{t+1}-f_\rho\|_K^2\right]=\bE_{Z^t}\left[\left\|-\prod_{i=1}^t(I-\eta_t L_K)f_\rho+\sum_{i=1}^t\prod_{j=i+1}^t(I-\eta_j L_K)\eta_i \mathcal{B}_i\right\|_K^2\right]\\
&&=\left\|\prod_{i=1}^t(I-\eta_t L_K)f_\rho\right\|_K^2+\bE_{Z^t}\left[\left\|\sum_{i=1}^t\prod_{j=i+1}^t(I-\eta_j L_K)\eta_i \mathcal{B}_i\right\|_K^2\right]\\
&&=\left\|\prod_{i=1}^t(I-\eta_t L_K)f_\rho\right\|_K^2+\sum_{i=1}^t\eta_i^2\bE_{Z^i}\left[\left\|\prod_{j=i+1}^t(I-\eta_j L_K) \mathcal{B}_i\right\|_K^2\right]\\
&&\le\left\|f_\rho\right\|_K^2+\sum_{i=1}^t\eta_i^2\left\|\prod_{j=i+1}^t(I-\eta_j L_K)\right\|_K^2 \bE_{Z^i}\left[\|\mathcal{B}_i\|_K^2\right]\\
&&\le\|f_\rho\|_K^2+\sum_{i=1}^t\eta_i^2 \kappa^2 \bE_{Z^{i-1}}\left[\mathcal{E}(f_i)\right].
\end{eqnarray*}
Then using the bound $\bE_{Z^{i-1}}\left[\mathcal{E}(f_i)\right]\le 20\|f_\rho\|_\rho^2+3\E(f_\rho)$ proved in Proposition 2 in \cite{Ying2008},
combining with the above estimate and part (b) of  Lemma \ref{lemma: eta bound} for $\frac12<\theta<1$, we have
\begin{equation*}
\bE_{Z^t}\left[\|f_{t+1}-f_\rho\|_K^2\right]\le \|f_\rho\|_K^2+\frac{2\theta\eta_1^2\kappa^2}{2\theta-1}\left(20\|f_\rho\|_\rho^2+3\E(f_\rho)\right).
\end{equation*}
And for $\theta=\frac12,$
\begin{equation*}
\bE_{Z^t}\left[\|f_{t+1}-f_\rho\|_K^2\right]\le \left(\|f_\rho\|_K^2+2\eta_1^2\kappa^2\left(20\|f_\rho\|_\rho^2+3\E(f_\rho)\right)\right)\ln t.
\end{equation*}
Finally, we complete the proof by the inequality $\bE_{Z^t}\left[\|f_{t+1}\|_K^2\right]\le 2\bE_{Z^t}\left[\|f_{t+1}-f_\rho\|_K^2\right]+2\|f_{\rho}\|_K^2.$
\end{proof}
Now we are ready to estimate the second part (\ref{equ: error decomposition}) in Proposition \ref{prposition: error decomposition in L2}.
\begin{proposition}\label{proposition: error 2}
Suppose Assumption \ref{assumption2} holds with $0\le\beta<1,$ and the step size $\eta_i=\eta_1 i^{-\theta}$ with $0<\eta_1<\frac1{\kappa^2}$ and $\frac12<\theta<1,$ one has
\begin{eqnarray*}
\sum_{i=1}^t \eta_i^2\left(2\kappa^2\bE_{Z^{i-1}}[\|f_i\|_K^2]+2M^2\right){\rm Tr}\left(L_K \prod_{j=i+1}^t(I-\eta_j L_K) \right)^2\le
C_{\theta,\beta}t^{-\min\{(2-\beta)(1-\theta),\theta\}}.
\end{eqnarray*}
\end{proposition}
\begin{proof}
The proof of the proposition can be divided into three steps as follows.

{\bf Step 1:} The choice of the step size $\eta_i$ and  Lemma \ref{lemma: ft bound} give that for $1 \leq i \leq t$,
\begin{eqnarray*}
2\kappa^2\bE_{Z^{i-1}}[\|f_i\|_K^2]+2M^2\le
C_\theta:=8\kappa^2\|f_\rho\|_K^2+\frac{8\theta\eta_1^2\kappa^2}{2\theta-1}\left(20\|f_\rho\|_\rho^2+3\E(f_\rho)\right)+2M^2.
\end{eqnarray*}
The property of trace ensures that if $A$ is an operator of trace class and $B$ is a bounded linear operator, there holds ${\rm Tr}(AB)\le {\rm Tr}(A)\|B\|$. Then for one can get from Assumption \ref{assumption2} that
\begin{eqnarray*}
{\rm Tr}\left(L_K^2 \prod_{j=i+1}^t(I-\eta_j L_K)^2 \right)\le {\rm Tr}(L_K^\beta)\left\|L_K^{2-\beta} \prod_{j=i+1}^t(I-\eta_j L_K)^2 \right\|.
\end{eqnarray*}
It follows that  (note that $\prod_{j=t+1}^t(I-\eta_jL_K)=I$)
\begin{eqnarray}\label{equ: error two}
&&\sum_{i=1}^t \eta_i^2\left(2\kappa^2\bE_{Z^{i-1}}[\|f_i\|_K^2]+2M^2\right){\rm Tr}\left(L_K^2 \prod_{j=i+1}^t(I-\eta_j L_K)^2 \right)\\
&&\le \eta_t^2C_\theta {\rm Tr}(L_K^2)+C_\theta {\rm Tr}\left(L_K^\beta \right) \sum_{i=1}^{t-1} \eta_i^2 \left\|L_K^{2-\beta}\prod_{j=i+1}^t(I-\eta_j L_K)^2
\right\|. \nonumber
\end{eqnarray}
{\bf Step 2:} We estimate the term $\left\|L_K^{2-\beta}\prod_{j=i+1}^t(I-\eta_j L_K)^2
\right\|.$ On one hand, using the elementary inequality $1-x\le \exp(-x)$ for $x>0,$ we obtain
\begin{eqnarray*}
&&\left\|L_K^{2-\beta}\prod_{j=i+1}^t(I-\eta_j L_K)^2
\right\|\le
\sup_{x>0}\left(x^{2-\beta}\prod_{j=i+1}^t(1-\eta_j x)^2
\right)\\
&&\le \sup_{x>0}\left(x^{2-\beta}\exp\left\{-2x\sum_{j=i+1}^t\eta_j
\right\}\right)\le  \left(\frac{2-\beta}{2e}\right)^{2-\beta} \frac{1}{\left(\sum_{j=i+1}^t\eta_j\right)^{2-\beta}}.
\end{eqnarray*}
The last inequality holds since the function $x^{2-\beta}\exp\left\{-2x\sum_{j=i+1}^t\eta_j
\right\}$  attains its maximum at $x=\frac{2-\beta}{2}\frac{1}{\sum_{j=i+1}^t\eta_j}.$

On the other hand, $L_K$ is a bounded operator with $\|L_K\|\le \kappa^2$, which implies
\begin{eqnarray*}
\left\|L_K^{2-\beta}\prod_{j=i+1}^t(I-\eta_j L_K)^2
\right\|\le \kappa^{4-2\beta}.
\end{eqnarray*}
As a consequence,
\begin{eqnarray*}
\left\|L_K^{2-\beta}\prod_{j=i+1}^t(I-\eta_j L_K)^2
\right\|&&\le \left(\left(\frac{2-\beta}{2e}\right)^{2-\beta} +\kappa^{4-2\beta}\right) \min\left\{1,\frac{1}{\left(\sum_{j=i+1}^t\eta_j\right)^{2-\beta}}\right\}\\
&&\le 2\left(\left(\frac{2-\beta}{2e}\right)^{2-\beta} +\kappa^{4-2\beta}\right) \frac{1}{1+\left(\sum_{j=i+1}^t\eta_j\right)^{2-\beta}}.
\end{eqnarray*}
Here we use $\min\{1/a,1/b\}\le \frac{2}{a+b},\forall a,b>0$ to get the last inequality.
Therefore,
\begin{equation}\label{equ: error2 part2}
\begin{split}
&\sum_{i=1}^{t-1} \eta_i^2 \left\|L_K^{2-\beta}\prod_{j=i+1}^t(I-\eta_j L_K)^2
\right\|\\
&\qquad \qquad\le
 2{\rm Tr}\left(L_K^\beta \right)  \left(\left(\frac{2-\beta}{2e}\right)^{2-\beta} +\kappa^{4-2\beta}\right) \sum_{i=1}^{t-1} \eta_i^2  \frac{1}{1+\left(\sum_{j=i+1}^t\eta_j\right)^{2-\beta}}.
\end{split}
\end{equation}
{\bf Step 3:}
It remains to estimate the sum on the right hand side of (\ref{equ: error2 part2}). To this end, we further divide the sum on the right hand side into two parts, which can be written as
\begin{eqnarray}\label{equ: sum}
\sum_{i=1}^{t-1} \eta_i^2  \frac{1}{1+\left(\sum_{j=i+1}^t\eta_j\right)^{2-\beta}}=\sum_{1 \le i\le \frac{t-1}{2}} \eta_i^2  \frac{1}{1+\left(\sum_{j=i+1}^t\eta_j\right)^{2-\beta}}+\sum_{\frac{t-1}{2}< i \leq t-1} \eta_i^2  \frac{1}{1+\left(\sum_{j=i+1}^t\eta_j\right)^{2-\beta}}.
\end{eqnarray}
We first consider the first part of (\ref{equ: sum}). Note that $i\le \frac{t-1}{2}$ implies $i+1\le \frac{t+1}{2}$. By the part (a) of Lemma \ref{lemma: eta bound}, there holds
\begin{eqnarray*}
&&\left(\sum_{j=i+1}^t\eta_j\right)^{2-\beta}\ge \left(\frac{\eta_1}{1-\theta}\right)^{2-\beta}
\left((t+1)^{1-\theta}-(i+1)^{1-\theta}\right)^{2-\beta}\\
&&\quad \ge \left(\frac{\eta_1}{1-\theta}\right)^{2-\beta} \left((t+1)^{1-\theta}-\left(\frac{t+1}{2}\right)^{1-\theta}\right)^{2-\beta}=\left(\frac{\eta_1(1-2^{\theta-1})}{1-\theta}\right)^{2-\beta} (t+1)^{(1-\theta)(2-\beta)}.
\end{eqnarray*}
This together with part (b) of Lemma \ref{lemma: eta bound} implies that
\begin{equation*}
\begin{split}
\sum_{1 \le i\le \frac{t-1}{2}} \eta_i^2  \frac{1}{1+\left(\sum_{j=i+1}^t\eta_j\right)^{2-\beta}}&\le (t+1)^{-(2-\beta)(1-\theta)}\left(\frac{1-\theta}{\eta_1(1-2^{\theta-1})}\right)^{2-\beta}\sum_{1 \le i\le \frac{t-1}{2}} \eta_i^2  \\
&\le t^{-(2-\beta)(1-\theta)}\left(\frac{1-\theta}{\eta_1(1-2^{\theta-1})}\right)^{2-\beta}\left(\frac{2\eta_1^2\theta}{2\theta-1}\right).
\end{split}
\end{equation*}
Now we turn to consider the  second part of (\ref{equ: sum}), i.e., the summation of $i>\frac{t-1}{2}$. When $x\in[i,i+1]$ and $i\le t-1,$ we have $$i^{-\theta}=i^{-\theta}(x+1)^{\theta}(x+1)^{-\theta}\le i^{-\theta}(i+2)^{\theta}(x+1)^{-\theta} \le 3(x+1)^{-\theta}.$$ It then follows that
\begin{equation*}
\begin{split}
\sum_{\frac{t-1}{2}< i \leq t-1} \eta_i^2  \frac{1}{1+\left(\sum_{j=i+1}^t\eta_j\right)^{2-\beta}}&\le \eta_1^2\sum_{\frac{t-1}{2}< i \leq t-1} \frac{i^{-2\theta}}{1+\left(\frac{\eta_1}{1-\theta}((t+1)^{1-\theta}-(i+1)^{1-\theta})\right)^{2-\beta}}\\
&\le \eta_1^2\left(\frac{t-1}{2}\right)^{-\theta}\sum_{\frac{t-1}{2}< i \leq t-1} \frac{i^{-\theta}}{1+\left(\frac{\eta_1}{1-\theta}((t+1)^{1-\theta}-(i+1)^{1-\theta})\right)^{2-\beta}}\\
&\le 3\eta_1^2\left(\frac{t-1}{2}\right)^{-\theta}\int_{\frac{t-1}{2}}^{t}   \frac{(x+1)^{-\theta}}{1+\left(\frac{\eta_1}{1-\theta}((t+1)^{1-\theta}-(x+1)^{1-\theta})\right)^{2-\beta}}dx.
\end{split}
\end{equation*}
What is left is to consider the integral, which can be estimated as
\begin{eqnarray*}
&&\int_{\frac{t-1}{2}}^{t}   \frac{(x+1)^{-\theta}}{1+\left(\frac{\eta_1}{1-\theta}((t+1)^{1-\theta}-(x+1)^{1-\theta})\right)^{2-\beta}}dx=-\frac{1}{\eta_1}\int_{\frac{t-1}{2}}^{t}   \frac{d\left(\frac{\eta_1}{1-\theta}((t+1)^{1-\theta}-(x+1)^{1-\theta})\right)}{1+\left(\frac{\eta_1}{1-\theta}((t+1)^{1-\theta}-(x+1)^{1-\theta})\right)^{2-\beta}}\\
&&=\frac{1}{\eta_1}\int_{0}^{\frac{\eta_1}{1-\theta}((t+1)^{1-\theta}-((t+1)/2)^{1-\theta})}   \frac{dy}{1+y^{2-\beta}}\le \frac{1}{\eta_1}\int_{0}^{\infty}   \frac{dy}{1+y^{2-\beta}}\\
&&=\frac{1}{\eta_1}\left(\int_{0}^1   \frac{dy}{1+y^{2-\beta}}+\int_{1}^{\infty}   \frac{dy}{1+y^{2-\beta}}\right)\le \frac{1}{\eta_1}\left(1+\frac{1}{1-\beta}\right)=\frac{2-\beta}{(1-\beta)\eta_1}.
\end{eqnarray*}Putting the above estimates back into (\ref{equ: sum}) yields that
\begin{eqnarray*}
&&\sum_{i=1}^{t-1} \eta_i^2  \frac{1}{1+\left(\sum_{j=i+1}^t\eta_j\right)^{2-\beta}}\le t^{-(2-\beta)(1-\theta)}\left(\frac{1-\theta}{\eta_1(1-2^{\theta-1})}\right)^{2-\beta}\left(\frac{2\eta_1^2\theta}{2\theta-1}\right)
+ 3\eta_1\left(\frac{t-1}{2}\right)^{-\theta}\frac{2-\beta}{1-\beta}\\
&&\le t^{-\min\{(2-\beta)(1-\theta),\theta\}}\left(\left(\frac{1-\theta}{1-2^{\theta-1}}\right)^{2-\beta}\left(\frac{2\eta_1^\beta\theta}{2\theta-1}\right)
+ 3\eta_1 4^{\theta} \frac{2-\beta}{1-\beta}\right).
\end{eqnarray*}
We finish our proof by combining the above estimate with (\ref{equ: error two}) and (\ref{equ: error2 part2}), and setting
$$C_{\beta,\theta}=C_\theta\left\{\eta_1^2 \kappa^4 +2{\rm Tr}\left(L_K^\beta \right)  \left(\left(\frac{2-\beta}{2e}\right)^{2-\beta} +\kappa^{4-2\beta}\right)\left(\left(\frac{1-\theta}{1-2^{\theta-1}}\right)^{2-\beta}\left(\frac{2\eta_1^\beta\theta}{2\theta-1}\right)
+ 3\eta_1 4^{\theta} \frac{2-\beta}{1-\beta}\right)\right\}.$$
\end{proof}
Now we are in a position to prove our first main result in Theorem \ref{theorem: rate in L2}.

\noindent{\bf Proof of Theorem \ref{theorem: rate in L2}}.
From Proposition \ref{proposition: error 1} and Proposition \ref{proposition: error 2}, we have
\begin{eqnarray*}
\bE_{Z^t}[\|f_{t+1}-f_\rho\|_\rho^2]\le \left(\left\|u_\rho\right\|_\rho^2 \left(\frac{r(1-\theta)}{e\eta_1(1-2^{\theta-1})}\right)^{2r}+C_{\beta,\theta}\right)t^{-\min\{2r(1-\theta),(2-\beta)(1-\theta),\theta\}}.
\end{eqnarray*}
One can easily see that
 \begin{enumerate}
 \item [(1)]when $2r>2-\beta,$ that is $2r+\beta>2,$ we choose $\theta=\frac{2-\beta}{3-\beta}.$ In this case, the rate is $\mathcal{O}(t^{-\frac{2-\beta}{3-\beta}}).$

\item [(2)]when $2r+\beta\le 2, $ we choose $\theta=\frac{2r}{2r+1}.$ Then the convergence rate is $\mathcal{O}(t^{-\frac{2r}{2r+1}}).$
\end{enumerate}
The proof is finished by setting $$C_{\theta,r,\beta}=\left(\left\|u_\rho\right\|_\rho^2 \left(\frac{r(1-\theta)}{e\eta_1(1-2^{\theta-1})}\right)^{2r}+C_{\beta,\theta}\right).$$
\qed

\section{Convergence Analysis in $\H_K$}\label{section: covergenceinHK}

As we mentioned, for the online learning algorithm (\ref{algorithm}) with polynomially decay step size, the capacity independent analysis in \cite{Ying2008} can not lead to converge in $\H_K$. In this section, we aim to fill this gap by the aid of capacity information in Assumption \ref{assumption2}. Similar to Proposition \ref{prposition: error decomposition in L2}, we establish the following error decomposition in $\H_K.$
\begin{proposition}\label{proposition: error decomposition in HK}
Let $\{f_t\}_{t\in\N}$ be defined by (\ref{algorithm}). Then $\bE_{Z^t}\left[\|f_{t+1}-f_\rho\|_K^2\right]$ can be bounded by
\begin{eqnarray}\label{equation: error decomposition in HK}
 \left\|\prod_{i=1}^t(I-\eta_t L_K)f_\rho\right\|_K^2 +\sum_{i=1}^t\eta_i^2\left(2\kappa^2\bE_{Z^{i-1}}[\|f_i\|_K^2]+2M^2\right){\rm Tr}\left(L_K\prod_{j=i+1}^t(I-\eta_j L_K)^2\right).
\end{eqnarray}
\end{proposition}
\begin{proof} By the same argument in the proof of Proposition \ref{prposition: error decomposition in L2}, there holds
\begin{eqnarray*}
&&\bE_{Z^t}\left[\|f_{t+1}-f_\rho\|_K^2\right]=\bE_{Z^t}\left[\left\|-\prod_{i=1}^t(I-\eta_t L_K)f_\rho+\sum_{i=1}^t\prod_{j=i+1}^t(I-\eta_j L_K)\eta_i \mathcal{B}_i\right\|_K^2\right]\\
&&=\left\|\prod_{i=1}^t(I-\eta_t L_K)f_\rho\right\|_K^2+\sum_{i=1}^t\eta_i^2\bE_{Z^{i-1}}\left[\left\|\prod_{j=i+1}^t(I-\eta_j L_K) \mathcal{B}_i\right\|_K^2\right]\\
&&\leq \left\|\prod_{i=1}^t(I-\eta_t L_K)f_\rho\right\|_K^2 +\sum_{i=1}^t\eta_i^2\left(2\kappa^2\bE_{Z^{i-1}}[\|f_i\|_K^2]+2M^2\right){\rm Tr}\left(L_K\prod_{j=i+1}^t(I-\eta_j L_K)^2\right).
\end{eqnarray*}
This completes the proof of Proposition \ref{proposition: error decomposition in HK}.
\end{proof}
Next we estimate the two terms in Proposition \ref{proposition: error decomposition in HK}. For the first part in (\ref{equation: error decomposition in HK}), we obtain similar bound as Proposition \ref{proposition: error 1} in the following.
\begin{proposition}\label{proposition: error 1 in HK}
Suppose that Assumption \ref{assumption1} holds with $r>\frac12,$ and $\eta_i=\eta_1 i^{-\frac12}$ with $0<\eta_1<\frac{1}{\kappa^2}$. Then there holds
\begin{eqnarray*}
\left\|\prod_{i=1}^t(I-\eta_i L_K)f_\rho\right\|_K^2\le \left\|u_\rho\right\|_\rho^2 \left(\frac{r}{e\eta_1(2-\sqrt{2})}\right)^{2r-1} t^{-\frac{1}{2}(2r-1)},
\end{eqnarray*}where $e$ is the base of the natural logarithm.
\end{proposition}
\begin{proof} Suppose that Assumption \ref{assumption1} holds with $r>\frac12$ and $\theta=\frac12$ in the part (a) of Lemma \ref{lemma: key lemma}, one has
\begin{eqnarray*}
&&\left\|\prod_{i=1}^t(I-\eta_i L_K)f_\rho\right\|_K\le \left\|\prod_{i=1}^t(I-\eta_i L_K)L_K^{r-\frac12}\right\| \|u_\rho\|_\rho\\
&&\le \sup_{x>0}\left(\prod_{i=1}^t(1-\eta_i x)x^{r-\frac12}\right)\left\|u_\rho\right\|_\rho\le \sup_{x>0}\left(\exp\left(-x\sum_{i=1}^t\eta_i\right)x^{r-\frac12}\right)\left\|u_\rho\right\|_\rho\\
&&\le \left(\frac{r}{e}\right)^{r-\frac12}(\sum_{i=1}^t\eta_i)^{-\left(r-\frac12\right)}\left\|u_\rho\right\|_\rho\le \left\|u_\rho\right\|_\rho \left(\frac{r}{e\eta_1(2-\sqrt{2})}\right)^{r-\frac12} t^{-\frac{1}{2}(r-\frac12)}.
\end{eqnarray*}
Thus we complete the proof.
\end{proof}
The estimate of the second part of (\ref{equation: error decomposition in HK}) is more involved.
\begin{proposition}\label{proposition: error 2 in HK}
Suppose that Assumption \ref{assumption2} holds with $0<\beta<1$, then one has
\begin{eqnarray*}
\sum_{i=1}^t\eta_i^2\left(2\kappa^2\bE_{Z^{i-1}}[\|f_i\|_K^2]+2M^2\right){\rm Tr}\left(L_K\prod_{j=i+1}^t(I-\eta_j L_K)^2\right) \le \tilde{C}_{\beta} t^{-\frac{1}{2}(1-\beta)}(\ln t)^2,
\end{eqnarray*}
where $\tilde{C}_{\beta}$ is independent of $t$ and will be given explicitly in the proof.
\end{proposition}
\begin{proof}
We first see from part (b) of Lemma \ref{lemma: ft bound} with $\theta=\frac12$ for the bounds  $\bE_{Z^{i-1}}[\|f_i\|_K^2]$ that
\begin{eqnarray*}
2\kappa^2\bE_{Z^{i-1}}[\|f_i\|_K^2]+2M^2\le C_\kappa\ln i,\quad \forall i\in\N,
\end{eqnarray*}
where $$C_\kappa=(8\kappa^2\|f_\rho\|_K^2+8\eta_1^2\kappa^4\left(20\|f_\rho\|_\rho^2+3\E(f_\rho)\right)+2M^2).$$
If Assumption \ref{assumption2} holds with $0<\beta<1,$ similar to the proof of Proposition \ref{proposition: error 2}, we have
\begin{eqnarray*}
&&{\rm Tr}\left(L_K\prod_{j=i+1}^t(I-\eta_j L_K)^2\right)\le {\rm Tr}(L_K^{\beta})\left\|L_K^{1-\beta}\prod_{j=i+1}^t(I-\eta_j L_K)^2\right\|\\
&&\le {\rm Tr}\left(L_K^\beta \right)  \left(\frac{1-\beta}{2e}\right)^{1-\beta}   \frac{1}{ \left(\sum_{j=i+1}^t\eta_j\right)^{1-\beta}}
\end{eqnarray*}
Therefore, note by $\prod_{j=t+1}^t(I-\eta_j L_K)=I,$ we have
\begin{eqnarray*}
&&\sum_{i=1}^t\eta_i^2\left(2\kappa^2\bE_{Z^{i-1}}[\|f_i\|_K^2]+2M^2\right){\rm Tr}\left(L_K\prod_{j=i+1}^t(I-\eta_j L_K)^2\right)\\
&&\le \eta_1^2\kappa^2 C_\kappa t^{-1}\ln t+C_\kappa {\rm Tr}\left(L_K^\beta \right)  \left(\frac{1-\beta}{2e}\right)^{1-\beta} \ln t\sum_{i=1}^{t-1} \eta_i^2  \frac{1}{ \left(\sum_{j=i+1}^t\eta_j\right)^{1-\beta}}.
\end{eqnarray*}
It remains to estimate the sum $ \sum_{i=1}^{t-1} \eta_i^2  \frac{1}{ \left(\sum_{j=i+1}^t\eta_j\right)^{1-\beta}}$, we also divide the sum into two parts, i.e.,
\begin{eqnarray*}
\sum_{i=1}^{t-1} \eta_i^2  \frac{1}{ \left(\sum_{j=i+1}^t\eta_j\right)^{1-\beta}}=\left(\sum_{1 \le i\le\frac{t-1}{2}}+\sum_{\frac{t-1}{2}<i\leq t-1}\right) \eta_i^2  \frac{1}{ \left(\sum_{j=i+1}^t\eta_j\right)^{1-\beta}},
\end{eqnarray*}
The first sum can be estimated similarly as in the proof of the Proposition \ref{proposition: error 2} as follows
\begin{eqnarray*}
\sum_{1 \le i\le\frac{t-1}{2}} \eta_i^2  \frac{1}{ \left(\sum_{j=i+1}^t\eta_j\right)^{1-\beta}}\le \left(\frac{1}{2-\sqrt2}\right)^{1-\beta}2\eta_1^{1+\beta}t^{-\frac12(1-\beta)}\ln t.
\end{eqnarray*}
The second part of the sum is distinct from that of Proposition \ref{proposition: error 2}. Since $0<1-\beta<1,$ then
\begin{eqnarray*}
&&\sum_{\frac{t-1}{2}< i \leq t-1} \eta_i^2  \frac{1}{ \left(\sum_{j=i+1}^t\eta_j\right)^{1-\beta}} \le \eta_1^2\left(\frac{t-1}{2}\right)^{-\frac12}\sum_{\frac{t-1}{2}< i \leq t-1} \frac{i^{-\frac12}}{ \left(2\eta_1((t+1)^{\frac12}-(i+1)^{\frac12})\right)^{1-\beta}}\\
&&\le 3\eta_1^2\left(\frac{t-1}{2}\right)^{-\frac12}\int_{\frac{t-1}{2}}^{t}   \frac{(x+1)^{-\frac12}}{ \left(2\eta_1((t+1)^{\frac12}-(x+1)^{\frac12})\right)^{1-\beta}}dx\\
&&\le -3\eta_1\left(\frac{t-1}{2}\right)^{-\frac12}\int_{\frac{t-1}{2}}^{t}   \frac{d\left(2\eta_1((t+1)^{\frac12}-(x+1)^{\frac12})\right)}{ \left(2\eta_1((t+1)^{\frac12}-(x+1)^{\frac12})\right)^{1-\beta}}\\
&&\le 6\eta_1 t^{-\frac12} \int_{0}^{2\eta_1(1-2^{-\frac12})(t+1)^{\frac12}}  \frac{dy}{y^{1-\beta}}\\
&&=6\eta_1 \frac{1}{\beta} \left(2\eta_1(1-2^{-\frac12})\right)^\beta t^{-\frac12}(t+1)^{\frac{\beta}{2}}\\
&&\le 6\eta_1^{1+\beta} \frac{1}{\beta} t^{-\frac12(1-\beta)}.
\end{eqnarray*}
Therefore,
\begin{eqnarray*}
\sum_{i=1}^{t-1} \eta_i^2  \frac{1}{ \left(\sum_{j=i+1}^t\eta_j\right)^{1-\beta}}\le
\left(\left(\frac{1}{2-\sqrt2}\right)^{1-\beta}2\eta_1^{1+\beta}
+
6\eta_1^{1+\beta} \frac{1}{\beta} \right)t^{-\frac12(1-\beta)}\ln t.
\end{eqnarray*}
Combining the above estimates together, we have
\begin{eqnarray*}
&&\sum_{i=1}^t\eta_i^2\left(2\kappa^2\bE_{Z^{i-1}}[\|f_i\|_K^2]+2M^2\right){\rm Tr}\left(L_K\prod_{j=i+1}^t(I-\eta_j L_K)^2\right)\\
&&\le \left(\eta_1^2\kappa^2 C_\kappa +C_\kappa {\rm Tr}\left(L_K^\beta \right)\left(\frac{1-\beta}{2e}\right)^{1-\beta} \left(\left(\frac{1}{2-\sqrt2}\right)^{1-\beta}2\eta_1^{1+\beta}
+
6\eta_1^{1+\beta} \frac{1}{\beta} \right)  \right)t^{-\frac12(1-\beta)}(\ln t)^2.
\end{eqnarray*}
Then the proof is finished by setting $$\tilde{C}_\beta=\left(\eta_1^2\kappa^2 C_\kappa +2C_\kappa {\rm Tr}\left(L_K^\beta \right)\left(\frac{1-\beta}{2e}\right)^{1-\beta} \left(\left(\frac{1}{2-\sqrt2}\right)^{1-\beta}2\eta_1^{1+\beta}
+
6\eta_1^{1+\beta} \frac{1}{\beta} \right)  \right).$$
\end{proof} At the end of this section, we prove Theorem \ref{theorem: rate in HK}, which presents the strong convergence in ${\cal H}_K$.

\noindent{\bf Proof of Theorem \ref{theorem: rate in HK}}. The proof is straightforward by combining the bounds of Proposition \ref{proposition: error 1 in HK} and Proposition \ref{proposition: error 2 in HK}, and setting $\widetilde{C}_{\theta,r,\beta}=\left\|u_\rho\right\|_\rho^2 \left(\frac{r}{e\eta_1(2-\sqrt{2})}\right)^{2r-1}+\tilde{C}_\beta.$

\end{document}